\documentclass[lettersize,journal]{IEEEtran}
\usepackage{amsmath,amsfonts}
\usepackage{algorithmic}
\usepackage{algorithm}
\usepackage{array}
\usepackage[caption=false,font=normalsize,labelfont=sf,textfont=sf]{subfig}
\usepackage{textcomp}
\usepackage{stfloats}
\usepackage{url}
\usepackage{verbatim}
\usepackage{graphicx}
\usepackage{cite}

\usepackage{mathtools}
\usepackage{bm}
\usepackage{amssymb}
\usepackage{xcolor}
\usepackage{xparse}
\usepackage{booktabs}
\usepackage{balance}

\usepackage{tikz}
\usetikzlibrary{arrows.meta,positioning}

\usepackage{pgfplots}
\pgfplotsset{compat=1.18}
\usepgfplotslibrary{statistics}

\usepackage{amsthm}
\theoremstyle{plain}
\newtheorem{proposition}{Proposition}
\newtheorem{corollary}{Corollary}

\graphicspath{{figures/}}

\newlength{\cmpwlen}
\newcommand{\cmpsetup}[1]{%
  \setlength{\cmpwlen}{#1\textwidth}%
  \setlength{\tabcolsep}{0pt}%
  \renewcommand{\arraystretch}{0.35}%
}

\newcommand{\cmpimg}[1]{\includegraphics[width=\cmpwlen]{#1}}

\newcommand{\cmppsnr}[2][\cmpwlen]{%
  \makebox[#1][c]{\tiny #2}%
}

\newcommand{\cmpoplabel}[1]{%
  \raisebox{1.15cm}{\rotatebox[origin=c]{90}{#1}}%
}

\newcommand{\cmprowSeven}[7]{%
  \cmpimg{#1/truth.png} &
  \cmpimg{#1/obs.png} &
  \cmpimg{#1/pigdm_nom.png} &
  \cmpimg{#1/pigdm_o1.png} &
  \cmpimg{#1/pigdm_o2.png} &
  \cmpimg{#1/dps.png} &
  \cmpimg{#1/bgdm.png} \\
  & \cmppsnr{#2} &
  \cmppsnr{#3} &
  \cmppsnr{#4} &
  \cmppsnr{#5} &
  \cmppsnr{#6} &
  \cmppsnr{#7} \\[-0.3pt]%
}

\NewDocumentCommand{\cmprowFive}{m m m m g}{%
  \IfNoValueTF{#5}{%
    \cmpimg{#1/truth.png} &
    \cmpimg{#1/obs.png} &
    \cmpimg{#1/pigdm_nom.png} &
    \cmpimg{#1/dps.png} &
    \cmpimg{#1/bgdm.png} \\
    & & \cmppsnr{#2} &
    \cmppsnr{#3} &
    \cmppsnr{#4} \\[1.5pt]%
  }{%
    \cmpimg{#1/truth.png} &
    \cmpimg{#1/obs.png} &
    \cmpimg{#1/pigdm_o2.png} &
    \cmpimg{#1/dps.png} &
    \cmpimg{#1/bgdm.png} \\
    & \cmppsnr{#2} &
    \cmppsnr{#3} &
    \cmppsnr{#4} &
    \cmppsnr{#5} \\[-0.3pt]%
  }%
}

\newcommand{\cmprowSevenOp}[8]{%
  \cmpoplabel{#1} &
  \cmpimg{#2/truth.png} &
  \cmpimg{#2/obs.png} &
  \cmpimg{#2/pigdm_nom.png} &
  \cmpimg{#2/pigdm_o1.png} &
  \cmpimg{#2/pigdm_o2.png} &
  \cmpimg{#2/dps.png} &
  \cmpimg{#2/bgdm.png} \\
  & & \cmppsnr{#3} &
  \cmppsnr{#4} &
  \cmppsnr{#5} &
  \cmppsnr{#6} &
  \cmppsnr{#7} &
  \cmppsnr{#8} \\[-0.3pt]%
}

\definecolor{cPrior}{HTML}{6B7280}  % slate gray
\definecolor{cMeas}{HTML}{A23B4E}   % brick
\definecolor{cScore}{HTML}{2B5D8A}  % steel blue

\newcommand{\xz}{\boldsymbol{x_0}}
\newcommand{\xt}{\boldsymbol{x_t}}
\newcommand{\xhz}{\boldsymbol{\hat x_0}}

\newcommand{\Real}{\mathbb{R}}
\newcommand{\Norm}{\mathcal{N}}
\newcommand{\Gam}{\mathcal{G}}
\newcommand{\Exp}{\mathbb{E}}
\newcommand{\Iden}{\boldsymbol{I}}

\newcommand{\abart}{\bar\alpha_t}
\newcommand{\eps}{\varepsilon}
\newcommand{\KL}{\mathrm{KL}}
\newcommand{\Fq}{\mathcal{F}}

\newcommand{\yy}{\boldsymbol{y}}
\newcommand{\Ab}{\boldsymbol{A}}

\def\BibTeX{{\rm B\kern-.05em{\sc i\kern-.025em b}\kern-.08em
    T\kern-.1667em\lower.7ex\hbox{E}\kern-.125emX}}

\begin{document}
\title{FB-GDM: Fully-Bayesian Guided Diffusion Models for High-Dimensional Linear Inverse Problems via Unsupervised Variational Inference}
\author{Gatien~Séguy and Thomas~Rodet%
\thanks{SATIE Laboratory, ENS Paris-Saclay, CNRS,
Université Paris-Saclay, 91190 Gif-sur-Yvette, France 
(e-mail: firstname.lastname@ens-paris-saclay.fr).}}

\markboth{IEEE Transactions on Image Processing}%
{Séguy \MakeLowercase{\textit{and}} Rodet: FB-GDM: Fully-Bayesian Guided Diffusion Models}

\maketitle

%=======================================
%+========== ABSTRACT =====================
%=======================================
\begin{abstract}
Diffusion models are powerful priors for linear inverse problems, but the reference guidance methods, Diffusion Posterior Sampling (DPS) and Pseudoinverse-Guided Diffusion Models ($\Pi$GDM), rely on scalar hyperparameters tuned per task, usually against the ground truth.
We introduce FB-GDM, a fully-Bayesian guided diffusion method that removes this calibration step.
Starting from the Gaussian approximation of $\Pi$GDM, we derive a closed-form conditional score that depends on two precision parameters (inverse variances), one associated with the denoising approximation and one with the observation likelihood, and treat them as latent variables inferred by variational inference at each reverse step.
A separable factorization makes each update scale linearly with the number of pixels, so the inference stays tractable at full image resolution, at a cost comparable to one $\Pi$GDM run.
FB-GDM requires neither the noise level nor the ground truth: its only inputs are the observation and the forward operator. Experiments on CelebA-HQ inverse problems establish two results. 
(i) The precision parameters, inferred from the observation alone, allow FB-GDM to outperform $\Pi$GDM at its nominal setting, even when the latter is given the true noise level, by up to 14~dB depending on the operator, and to match the ground-truth-calibrated $\Pi$GDM oracle within 0.1~dB.
(ii) FB-GDM is robust when the forward operator, the noise level, or the image distribution changes: it stays close to a per-problem $\Pi$GDM oracle throughout and does not exhibit the hallucinations observed with DPS, whereas DPS substantially degrades at a fixed scale and $\Pi$GDM stays competitive only if it is re-tuned against the ground truth for each new problem.
When the prior is applied to images outside its training set, this re-balancing between data and prior keeps FB-GDM faithful where
a fixed face-prior guidance can otherwise hallucinate.
\end{abstract}
\begin{IEEEkeywords}
High-dimensional linear inverse problems, diffusion models, score-based generative models, Bayesian inference, variational inference, image reconstruction, unsupervised hyperparameter estimation, out-of-distribution robustness.
\end{IEEEkeywords}
%=======================================
%+========== INTRO =====================
%=======================================
\section{Introduction}
\IEEEPARstart{D}{iffusion} models have established themselves as powerful unconditional image generators \cite{dhariwal2021diffusion,chen2024overview}. They define a noise chain allowing them to progressively move from natural images to an image containing only additive white Gaussian noise, then learn to reverse this chain by estimating the score $\nabla_{\xt} \log p(\xt)$ through a denoising loss \cite{hyvarinen2005estimation, vincent2011connection, ho2020ddpm,song2021sde}. Generation is then performed by numerically integrating the associated reverse stochastic differential equation \cite{song2021sde,karras2022elucidating}. More precisely, a neural network is trained to predict the noise added to the image at each step, starting from a training set of natural  images to which noise is gradually added. This diffusion model has made it possible to generate images of a quality previously out of reach for earlier generative models (Gaussian mixtures, GANs, variational autoencoders). Diffusion models now underpin the leading text-to-image systems \cite{rombach2022stablediff, ramesh2022dalle2, saharia2022imagen,dhariwal2021diffusion} and their multimodal extensions to video \cite{ho2022videodm} and audio \cite{kong2021diffwave}.\\
\indent The same ability to model complex image distributions
motivates a second, more demanding use: as a prior in a linear inverse
problem. The central question is then the trade-off between the
information carried by the data, through the likelihood, and the
information learned from training, through the prior; the goal is to
set this trade-off automatically, from the observation alone.
Diffusion priors have already been deployed across scientific imaging
modalities, such as accelerated MRI \cite{jalal2021robust,
chung2022scoremri, song2022solving} or astronomical imaging
\cite{feng2023score, sun2024deep}. These diffusion-based inversion
methods, called conditional approaches, require the conditional score. This score has no exact expression, and several methods approximate it \cite{chung2023dps, song2023pigdm, kawar2022ddrm, wang2023ddnm, zhu2023diffpir, mardani2024reddiff, chung2022mcg, rout2023psld, feng2023score}. DPS \cite{chung2023dps} and $\Pi$GDM \cite{song2023pigdm} have recently established themselves as two reference approaches.\\
\indent However, both methods share a practical limitation. Their performance depends critically on scalar hyperparameters tuned case by case: DPS introduces $\zeta'$, $\Pi$GDM the pair $(\sigma_b^2, w)$. These settings vary with the task, the noise level and the data, with no calibration procedure in the absence of ground truth, a point raised by several works \cite{chung2023dps, song2023pigdm, pandey2024fast, feng2023score}.\\
\indent In the scientific settings that motivate these priors, such as medical or astronomical imaging, the ground truth is often unavailable and the noise level is rarely known in advance: a method tuned against them cannot be deployed there. A principled alternative evaluates the diffusion prior exactly through the probability-flow ODE~\cite{feng2023score}: it is robust out of distribution, but each evaluation requires an ODE solve, which restricts the approach to low-dimensional problems.\\
\indent Tuning hyperparameters without ground truth is, however, not a new problem. Classical inverse problems offer several possible approaches to estimate their hyperparameters: the L-curve \cite{hansen1992lcurve}, the generalized cross-validation \cite{golub1979cv, golub1997generalized}, or the maximization of the marginal likelihood \cite{mackay1992bayesian, molina1999bayesian}. More recently, fully Bayesian approaches \cite{giovannelli2007unsupervised,mohammad1996full,chaari2011parameter,chouzenoux2024sparse,dobigeon2009joint,woodbury2000full} treat the hyperparameters as random variables and estimate them jointly with the target. The joint posterior distribution then has no explicit form. It must therefore be approximated. Bayesian variational inference is one of these approximations \cite{quinn2006variational}, and fast variants exist \cite{fraysse2014measure, zheng2015efficient}.\\
\indent For solving linear inverse problems, data-driven approaches yield excellent results but are difficult to interpret in terms of information. Classical Bayesian inversion provides theoretical guarantees, but its priors are too generic to encode complex image structure. Conditional diffusion models offer the best of both worlds.
We propose FB-GDM (Fully-Bayesian Guided Diffusion Models), a generalization of $\Pi$GDM \cite{song2023pigdm}, one of the strongest and most widely used guidance methods for linear inverse problems. Its hyperparameters can be interpreted as precisions attached to the Gaussian denoising approximation and to the Gaussian guidance likelihood. FB-GDM treats these two precision parameters (the inverse variances of the denoising approximation and of the observation likelihood) as unknown random variables, and estimates them at each reverse step with the fast variational inference of~\cite{zheng2015efficient}, instead of fixing them by hand.\\
\indent We aim for guidance that needs no task-specific tuning. For this, we establish a new closed-form expression of the posterior score. At each reverse step, the Gaussian denoising approximation and the Gaussian observation likelihood define two local sources of information: one coming from the learned diffusion prior through $\xhz(\xt)$, and one coming from the measurements through $\yy$. FB-GDM automatically balances these two sources by estimating both precisions from the data, with no value set by hand. The resulting procedure requires no task-specific tuning: the user provides $\yy$ and $\Ab$, nothing else, and needs no expertise in variational inference.\\
\indent Moreover, unlike exact-prior evaluation~\cite{feng2023score}, this approach scales to high-dimensional images: a separable factorization (Section~\ref{sec:methode:separable})  makes the cost of each variational update grow only linearly with the number of pixels, so one FB-GDM run costs about as much as a single $\Pi$GDM run at fixed parameters.\\
\indent Finally, the automatic balance between the data and the contribution of the prior limits hallucinations in cases outside the training distribution.\\
\indent The rest of the paper is organized as follows. Section~\ref{sec:background} provides background on inverse problems, diffusion models and their conditional use, including $\Pi$GDM. Section~\ref{sec:methode} develops the proposed method: the closed-form expression of the conditional score, followed by its variational Bayesian estimation. Section~\ref{sec:experiments} reports the experiments and shows that the method stays close to a $\Pi$GDM oracle tuned on the ground truth in distribution, while remaining robust to changes of operator, noise level, and image distribution.
\section{Background}
\label{sec:background}
\subsection{Inverse problems}
We consider a linear inverse problem corrupted by additive white Gaussian noise. Therefore, we have 
\begin{equation}
\yy = \Ab \xz + \boldsymbol{b},
\qquad \boldsymbol{b}\sim\Norm(0,\sigma_b^2\Iden),
\label{eq:inverse_pb}
\end{equation}
with $\Ab \in \Real^{m \times n}$ known. The scalar $\sigma_b^2$ denotes the true observation noise variance. Reconstructing $\xz$ is ill-posed and requires a prior, through $p(\xz \mid \yy) \propto p(\yy \mid \xz)\, p(\xz)$. Classical priors (total variation, sparsity) are not very expressive. We instead learn $p(\xz)$ with a generative model. Diffusion models dominate this approach and serve here as the prior.\\
\subsection{Diffusion models}
A diffusion model transforms a distribution that is easy to sample from, a standard Gaussian, into the complex distribution of natural images~\cite{song2021sde}. In continuous time, a forward stochastic differential equation (SDE) gradually adds noise to an image $\xz$ from $t=0$ to $t=T$:
\begin{equation}
\mathrm{d}\xt = f(\xt, t)\,\mathrm{d}t + g(t)\,\mathrm{d}w,
\label{eq:fwd_sde}
\end{equation}
where $f$ is the drift, $g$ the diffusion coefficient and $w$ a standard Wiener process. The distribution $p_T$ is a standard Gaussian.\\
DDPM \cite{ho2020ddpm} is the variance preserving (VP) discretization of \eqref{eq:fwd_sde}: the forward process becomes a Markov chain of Gaussian steps. After marginalisation we can express $\xt$ with respect to $\xz$ : 
\begin{equation}
\xt = \sqrt{\bar\alpha_t}\, \xz + \sqrt{1 - \bar\alpha_t}\, \boldsymbol{\varepsilon}, \quad \boldsymbol{\varepsilon} \sim \mathcal{N}(0, I),
\label{eq:fwd}
\end{equation}
with $\bar\alpha_t$ decreasing from $1$ to $0$.\\
A network $\varepsilon_\theta(\xt, t)$ predicts the noise added at each diffusion step and gives the score \cite{ho2020ddpm, song2021sde}:
\begin{equation}
\nabla_{\xt} \log p_t(\xt) = -\varepsilon_\theta(\xt, t) / \sqrt{1 - \bar\alpha_t}.
\label{eq:score}
\end{equation}
Tweedie's formula provides the posterior mean
\begin{equation}
\xhz(\xt) = \mathbb{E}_{\xz\mid \xt}[\xz] = \frac{\xt - \sqrt{1 - \bar\alpha_t}\, \varepsilon_\theta(\xt, t)}{\sqrt{\bar\alpha_t}}.
\label{eq:tweedie}
\end{equation}
Generation amounts to traversing time backward. Anderson's reverse SDE \cite{anderson1982reverse} brings out the score as the only data-dependent term:
\begin{equation}
\mathrm{d}\xt = \big[f(\xt, t) - g(t)^2\,\nabla_{\xt}\log p_t(\xt)\big]\,\mathrm{d}t + g(t)\,\mathrm{d}\bar w,
\label{eq:rev_sde}
\end{equation}
where $\bar w$ is a standard Wiener process in reversed time. Its VP discretization, starting from $x_T \sim \mathcal{N}(0, I)$, gives the DDPM reverse step. We first write it using the score:
\begin{equation}
\mathbf{x_{t-1}} = \frac{1}{\sqrt{\alpha_t}}\Big(\xt + \beta_t\,\nabla_{\xt}\log p_t(\xt)\Big) + \sqrt{\tilde\beta_t}\, \boldsymbol{z},
\quad \boldsymbol{z} \sim \mathcal{N}(0, \boldsymbol{I}),
\label{eq:ddpm_score}
\end{equation}
then, by injecting \eqref{eq:score}:
\begin{equation}
\mathbf{x_{t-1}} = \frac{\sqrt{\bar\alpha_{t-1}}\,\beta_t}{1 - \bar\alpha_t}\, \xhz + \frac{\sqrt{\alpha_t}\,(1 - \bar\alpha_{t-1})}{1 - \bar\alpha_t}\, \xt + \sqrt{\tilde\beta_t}\, \boldsymbol{z},
\label{eq:ddpm}
\end{equation}
with $\alpha_t = \bar\alpha_t / \bar\alpha_{t-1}$, $\beta_t = 1 - \alpha_t$, $\tilde\beta_t = \tfrac{1 - \bar\alpha_{t-1}}{1 - \bar\alpha_t}\beta_t$, $\boldsymbol{z} \sim \mathcal{N}(0, \boldsymbol{I})$.
\subsection{Conditional diffusion models}
Generating according to $p(\xz \mid \yy)$ amounts to replacing the score \eqref{eq:score} with the conditional score in the reverse step \eqref{eq:ddpm_score}:
\begin{equation}
\mathbf{x_{t-1}} = \frac{1}{\sqrt{\alpha_t}}\Big(\xt + \beta_t\,\nabla_{\xt}\log p_t(\xt\mid \yy)\Big) + \sqrt{\tilde\beta_t}\, \boldsymbol{z},
\label{eq:ddpm_score_conditionnel}
\end{equation}
with $\boldsymbol{z} \sim \mathcal{N}(0, \boldsymbol{I})$.
Bayes' rule decomposes this score into two terms:
\begin{equation}
\nabla_{\xt} \log p_t(\xt \mid \yy) = \underbrace{\nabla_{\xt} \log p_t(\xt)}_{\text{prior score}} + \underbrace{\nabla_{\xt} \log p_t(\yy \mid \xt)}_{\text{measurement term}}.
\label{eq:bayes_score}
\end{equation}
The first term, the prior score, is provided by the network through~\eqref{eq:score}. The second, the measurement term, has no exact
expression: it involves
\begin{equation}
p_t(\yy \mid \xt) = \int p(\yy \mid \xz)\, p_t(\xz \mid \xt)\, dx_0,
\label{eq:meas}
\end{equation}
intractable because the distribution $p_t(\xz \mid \xt)$ is unknown. Methods differ in their approximation of $p_t(\xz \mid \xt)$, hence of this measurement term.
\subsubsection{DPS}
DPS \cite{chung2023dps} takes $p_t(\xz \mid \xt) \approx \delta(\xz - \xhz(\xt))$, hence $p_t(\yy \mid \xt) \approx p(\yy \mid \xhz(\xt))$ and
\begin{equation}
\nabla_{\xt} \log p_t(\yy \mid \xt,\zeta') \approx -\zeta'\, \nabla_{\xt} \|\yy - \Ab\xhz(\xt)\|^2.
\label{eq:dps}
\end{equation}
The scale factor $\zeta'$, which absorbs the noise level, is tuned per task and per noise level.\\
\subsubsection{$\Pi$GDM}
$\Pi$GDM \cite{song2023pigdm} relies on two Gaussian approximations, each carrying a variance that is a hyperparameter.
\paragraph*{The denoising approximation} The inverse distribution $p_t(\xz \mid \xt)$ is replaced by an isotropic Gaussian centered on the Tweedie estimate \eqref{eq:tweedie}:
\begin{equation}
p_t(\xz \mid \xt,\, r_t^2) \approx \Norm\!\big(\xhz(\xt),\, r_t^2\,\Iden\big).
\label{eq:pigdm_approx}
\end{equation}
The variance $r_t^2$ measures the assumed dispersion of $\xz$ around $\xhz$.
\paragraph*{Observation model}
In the linear Gaussian inverse-problem setting of \eqref{eq:inverse_pb}, the true observation likelihood is Gaussian, with noise variance $\sigma_b^2$. In $\Pi$GDM, the same Gaussian likelihood is used for guidance, with $\sigma_b^2$ assumed to be known. In practice, it is chosen by the user or tuned on a ground-truth criterion. With \eqref{eq:pigdm_approx}, the Gaussian marginalization of \eqref{eq:meas} is closed and brings out the two hyperparameters:
\begin{equation}
p_t(\yy \mid \xt,\, \sigma_b^2, r_t^2)
\approx
\Norm\!\big(\Ab\xhz,\, \sigma_b^2\,\Iden
+ r_t^2\,\Ab\Ab^\top\big).
\label{eq:pigdm_law}
\end{equation}
The measurement term follows in pseudo-inverse form:
\begin{equation} 
\begin{split}
&\nabla_{\xt}\log p_t(\yy \mid \xt,\, \sigma_b^2, r_t^2) \approx \\ &\Big( (\yy - \Ab\xhz)^\top (r_t^2 \Ab\Ab^\top + \sigma_b^2\Iden)^{-1} \Ab\, \tfrac{\partial \xhz}{\partial \xt} \Big)^\top,
\label{eq:pigdm_grad}
\end{split}
\end{equation}
evaluated as a Jacobian-vector product by backpropagation through $\xhz$. The pseudo-inverse term requires solving the system $(r_t^2 \Ab\Ab^\top + \sigma_b^2\Iden)\,u = \yy - \Ab\xhz$, by conjugate gradient for example.\\
\paragraph*{Guidance weighting and update}
The prior score $\nabla_{\xt}\log p_t(\xt)$ is given by the network through \eqref{eq:score}. The measurement term \eqref{eq:pigdm_grad} is not added as is. Following classifier guidance \cite{dhariwal2021diffusion}, $\Pi$GDM weights it by three factors: the adaptive weight $r_t^2$ (Section~3.3 of \cite{song2023pigdm}), the factor $\sqrt{\abart}$ specific to the VP parameterization, and a scalar scale $w$ ($w=1$ nominal, tuned per task, up to the grid search of \cite{pandey2024fast}). This weighted correction is added in sample space to an unconditional step, a DDIM step in the original formulation (eq.~10 of \cite{song2023pigdm}). We work in the DDPM sampler, used throughout the article, whose unconditional step is \eqref{eq:ddpm_score}. The conditional update then keeps the form of \eqref{eq:ddpm_score_conditionnel}, with the conditional score identified as:
\begin{align}
\nabla_{\xt}\log p_t(\xt \mid \yy) =& \frac{w\sqrt{\alpha_t\abart}\,r_t^2}{\beta_t}\,\nabla_{\xt}\log p_t(\yy \mid \xt,\, \sigma_b^2, r_t^2)\notag\\
&+ \nabla_{\xt}\log p_t(\xt).
\label{eq:pigdm_score}
\end{align}
With $\Pi$GDM's choice $r_t^2 = 1-\abart$, this weighting coefficient equals $w\,\sqrt{\alpha_t\abart}\,(1-\abart)/\beta_t$.\\
Algorithm~\ref{alg:pigdm} details one such reverse step.
\begin{algorithm}[h]
\caption{$\Pi$GDM \cite{song2023pigdm}, one reverse step in the DDPM sampler}
\label{alg:pigdm}
\begin{algorithmic}[1]
\REQUIRE $\xt$, $\yy$, $\Ab$, network $\eps_\theta$, $\bar\alpha_t$; \textbf{hyperparameters} $\sigma_b^2,\ w$
\STATE \textbf{$\triangleright$ Score computation}
\STATE $r_t^2 \gets 1 - \bar\alpha_t$
\STATE $\eps_\theta \gets \eps_\theta(\xt, t)$ \COMMENT{network}
\STATE $\xhz \gets (\xt - \sqrt{1-\abart}\,\eps_\theta)/\sqrt{\abart}$ \COMMENT{Tweedie, \eqref{eq:tweedie}}
\STATE $\text{score}({\xt}) \gets -\,\eps_\theta/\sqrt{1-\abart}$ \COMMENT{\eqref{eq:score}}
\STATE $\text{score}({\yy\mid \xt}) \gets \nabla_{\xt}\log p_t(\yy \mid \xt,\, \sigma_b^2, r_t^2)$ \COMMENT{\eqref{eq:pigdm_grad} via \eqref{eq:pigdm_law}}
\STATE $\text{score}({\xt\mid \yy})  \gets \text{score}({\xt})  + \frac{w\,\sqrt{\alpha_t\abart}\,r_t^2}{\beta_t}\, \text{score}({\yy\mid \xt}) $ 
\STATE \textbf{$\triangleright$ DDPM reverse step}
\STATE $\bm z \sim \Norm(0, \Iden)$
\STATE $\bm x_{t-1} \gets \tfrac{1}{\sqrt{\alpha_t}}\Bigl(\xt + \beta_t\, \text{score}(\xt\mid \yy) \Bigr)+ \sqrt{\tilde\beta_t}\, \boldsymbol{z}$ 
\RETURN $\bm x_{t-1}$
\end{algorithmic}
\vspace{-3pt}
\end{algorithm}
These three scalars $(r_t^2, \sigma_b^2, w)$ control the $\Pi$GDM guidance. $\Pi$GDM sets $r_t^2=1-\bar\alpha_t$, a data-agnostic heuristic (Appendix A.3 of \cite{song2023pigdm}), while $\sigma_b^2$ and $w$ are chosen per task.
\section{Method}
\label{sec:methode}
The goal is to express $\nabla_{\xt}\log p_t(\xt\mid \yy)$ in closed form and to estimate its uncertainty parameters by variational Bayesian inference. At each reverse step, inference is carried out at fixed $\xt$: $\xt$ acts only as a conditioning variable and is never inferred. The inferred precisions adapt the balance between the measurements and the Tweedie prediction within the posterior mean, while the deterministic VP-dependent scale of the score correction is retained.
\subsection{Closed form conditional score}
\label{sec:methode:general}
We first express the conditional score from the Gaussian approximation of the denoising. We denote by $\xhz(\xt)$ the Tweedie estimate \eqref{eq:tweedie}, and we reuse the Gaussian approximation \eqref{eq:pigdm_approx} of $\Pi$GDM \cite{song2023pigdm}, denoted (H1):
\begin{equation}
\bigl( \xz \mid \xt,\, r_t^2\bigr) \sim \Norm\!\left(\xhz(\xt),\, r_t^2\, \Iden\right).
\tag{H1}\label{eq:H1}
\end{equation}
The Markov chain $\xt \to \xz \to \yy$ in Figure~\ref{fig:graphical_model} entails the conditional independence $\yy \perp\!\!\!\perp \xt \mid \xz$.

\begin{proposition}[General conditional score]
\label{prop:score_general}
Under assumption \eqref{eq:H1} and the conditional independence $\yy \perp\!\!\!\perp \xt \mid \xz$, the conditional score reads
\begin{equation}
\begin{aligned}
\nabla_{\xt}\log p_t(\xt \mid \yy) =  &\frac{1}{r_t^2}\Big(\frac{\partial \xhz}{\partial \xt}\Big)^{\!\top}\!\big(\Exp_{\xz \mid \xt, \yy}[\xz] - \xhz(\xt)\big)\\ &+\nabla_{\xt}\log p_t(\xt).
\end{aligned}
\label{eq:score_general}
\end{equation}
\end{proposition}
\begin{proof}
See Appendix~\ref{app:closed_form_score1}.
\end{proof}
The first term is the measurement term: the gap $\Exp_{\xz \mid \xt, \yy}[\xz] - \xhz$ between the posterior mean and the denoiser estimate, transported into the space of $\xt$ by $\partial\xhz/\partial\xt$. The problem reduces to computing the posterior mean $\Exp_{\xz \mid \xt, \yy}[\xz]$.
The second is the prior score \eqref{eq:score}, provided by the network.
\begin{corollary}[Gaussian closed form]
\label{cor:closed_form}
For a Gaussian observation likelihood with working variance $\sigma_b^2$, and under the denoising approximation \eqref{eq:H1}, the posterior distribution $p_t(\xz \mid \xt, \yy)$ is Gaussian,
\begin{equation}
(\xz \mid \xt, \yy) \sim \Norm(\bm\mu_{\text{post}},\, \bm\Sigma_{\text{post}}),
\label{eq:gauss_post}
\end{equation}
\begin{align}
\bm\Sigma_{\text{post}} &= \left(\frac{\Ab^\top \Ab}{\sigma_b^2} + \frac{\Iden}{r_t^2}\right)^{\!-1},
\label{eq:sigma_post}\\
\bm\mu_{\text{post}} &= \bm\Sigma_{\text{post}}\!\left(\frac{\Ab^\top \yy}{\sigma_b^2} + \frac{\xhz(\xt)}{r_t^2}\right),
\label{eq:mu_post}
\end{align}
and the conditional score \eqref{eq:score_general} admits the closed form
\begin{equation}
\begin{aligned}
\nabla_{\xt}\log p_t(\xt \mid \yy) = &\frac{1}{r_t^2}\Big(\frac{\partial \xhz}{\partial \xt}\Big)^{\!\top}\!\big(\bm\mu_{\text{post}} - \xhz(\xt)\big) \\&+\nabla_{\xt}\log p_t(\xt) 
\end{aligned}
\label{eq:score_closed}
\end{equation}
\end{corollary}
\begin{proof}
See Appendix~\ref{app:closed_form_score2}.
\end{proof}
\noindent $\bm\mu_{\text{post}}$ is a weighted trade-off between $\Ab^\top \yy / \sigma_b^2$ (consistency with the measurements) and $\xhz(\xt) / r_t^2$ (prediction of the denoiser). When $\sigma_b \to 0$ and $\Ab^\top \Ab$ invertible, $\bm\mu_{\text{post}} \to (\Ab^\top \Ab)^{-1}\Ab^\top \yy$, the least-squares solution. When $r_t \to 0$, $\bm\mu_{\text{post}} \to \xhz$.
In FB-GDM, the inferred precisions determine $\bm\mu_{\text{post}}$ and therefore the balance between the Tweedie prediction and the measurements. Only in \eqref{eq:score_closed}, the outer factor $r_t^{-2}$ is replaced by $(1-\bar\alpha_t)^{-1}$, preserving the VP time scaling without introducing an additional guidance weight.\\
Substituting \eqref{eq:score_closed} into the DDPM reverse update \eqref{eq:ddpm_score_conditionnel} yields a step identical to the unconditional one, with the prior score corrected by the measurement term. The structure recalls $\Pi$GDM's guidance \eqref{eq:pigdm_score}, but without the scalar scale $w$: the inferred precisions determine the posterior mean, while the outer VP-dependent factor remains fixed. This connection is made precise in Appendix~\ref{app:link}.
\subsection{Hierarchical Bayesian model}
\label{sec:methode:hierarchique}
Expressions \eqref{eq:sigma_post}--\eqref{eq:mu_post} depend on the working variances $r_t^2$ and $\sigma_b^2$. In FB-GDM, neither of them is fixed from the true noise level nor tuned from the ground truth. Instead, we treat the corresponding precisions as unknown random variables and infer them at each reverse diffusion step to determine the posterior mean.\\
The Gamma family is conjugate to the Gaussians for the precision, not for the variance. We therefore parameterize:
\begin{equation}
\gamma_r \coloneqq 1/r_t^2, \qquad \gamma_b \coloneqq 1/\sigma_b^2.
\label{eq:precisions}
\end{equation}
The variables to infer form the vector $\mathbf{v} = (\xz, \gamma_r, \gamma_b)$; $\xt$ remains a conditioning, which enters only through $\xhz(\xt)$.
\begin{figure}[h]
\centering
\begin{tikzpicture}[
  >=Stealth,
  lat/.style={circle, draw=red!70!black, fill=red!12, line width=0.6pt, minimum size=8mm, inner sep=1pt},
  hyp/.style={circle, draw=red!70!black, fill=red!12, line width=0.6pt, minimum size=6.5mm, inner sep=0pt},
  obs/.style={circle, draw=blue!70!black, fill=blue!18, line width=0.6pt, minimum size=8mm, inner sep=1pt},
  fix/.style={rectangle, rounded corners=1.5pt, draw=black!55, fill=black!8, line width=0.6pt, minimum size=7.5mm, inner sep=2pt},
  lbl/.style={font=\scriptsize, inner sep=1pt},
  grp/.style={draw, dashed, rounded corners, line width=0.6pt}
]
\node[lat] (x0)   at (1.8,1.5)  {$\xz$};
\node[obs] (y)    at (0,0)      {$\yy$};
\node[fix] (xt)   at (3.6,0)    {$\xt$};
\node[hyp] (taur) at (1.8,2.6)  {$\gamma_r$};
\node[hyp] (taub) at (-1.2,0)   {$\gamma_b$};
% precisions: straight arrows
\draw[->] (taur) -- (x0);
\draw[->] (taub) -- (y);
% generative factors: arcs to the outside
\draw[->] (xt) to[bend right=20] node[lbl, sloped, above] {$p(\xz\mid\xt,\gamma_r)$} (x0);
\draw[->] (x0) to[bend right=20] node[lbl, sloped, above] {$p(\yy\mid\xz,\gamma_b)$} (y);
% intractable posterior
\draw[->, blue!70!black, densely dotted, line width=0.7pt]
  (y) to[bend right=28] node[lbl, sloped, below, text=blue!70!black] {$p(\xz\mid\xt,\yy)$} (x0);

\draw[grp] (-0.7,-0.5) rectangle (4.1,2.15);

\end{tikzpicture}
\caption{Graphical model at fixed $\xt$. Red circles: inferred variables ($\xz,\gamma_r,\gamma_b$); blue circle: observation $\yy$; square: fixed conditioning $\xt$. Solid line: tractable generative factors. Blue dashed line: posterior $p(\xz\mid\xt,\yy)$, intractable, the target of inference.}
\label{fig:graphical_model}
\end{figure}
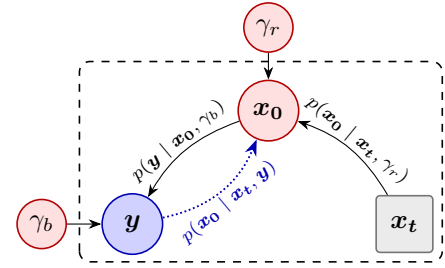
At fixed $\xt$, the independences of the model (Figure~\ref{fig:graphical_model}) factorize the joint distribution, from general to particular:
\begin{equation}
p_t(\mathbf{v}, \yy \mid \xt) = p(\yy \mid \xz, \gamma_b)\, p(\xz \mid \gamma_r, \xt)\, p(\gamma_r)\, p(\gamma_b),
\label{eq:joint}
\end{equation}
whose factors are:
\begin{align}
\xz \mid \gamma_r, \xt &\sim \Norm\!\left(\xhz(\xt),\, \gamma_r^{-1}\, \Iden\right),
\tag{M1}\label{eq:M1}\\
\yy \mid \xz, \gamma_b &\sim \Norm\!\left(\Ab\xz,\, \gamma_b^{-1}\, \Iden\right),
\tag{M2}\label{eq:M2}\\
 p(\gamma_r) &\propto \gamma_r^{-1},
\tag{M3}\label{eq:M3}\\
 p(\gamma_b) &\propto \gamma_b^{-1}.
\tag{M4}\label{eq:M4}
\end{align}
The factors \eqref{eq:M1} and \eqref{eq:M2} correspond respectively to the denoising approximation \eqref{eq:H1} and to a Gaussian working likelihood
built from the forward model in \eqref{eq:inverse_pb}; their precisions are now treated as random variables. The priors \eqref{eq:M3} and \eqref{eq:M4} are the Jeffreys priors of the scale parameters \cite{jeffreys1946}: non-informative, improper, and with no parameter to tune. 
Expanding each factor of \eqref{eq:joint} and keeping only the terms which depend on $\mathbf{v}$:
\begin{align}
&\log p_t(\mathbf{v}, \yy \mid \xt)
= \left(\tfrac{m}{2} - 1\right)\log\gamma_b\notag\\
&- \tfrac12\|\yy - \Ab\xz\|^2\,\gamma_b
+ \left(\tfrac{n}{2} - 1\right)\log\gamma_r\notag\\
&- \tfrac12\|\xz - \xhz(\xt)\|^2\,\gamma_r + \text{const.}
\label{eq:logjoint}
\end{align}
The structure, quadratic in $\xz$ and linear in $\gamma_b, \log\gamma_b, \gamma_r, \log\gamma_r$, guarantees Gaussian and Gamma conjugacy.
\subsection{Variational Bayesian inference}
\label{sec:methode:vb}
The posterior distribution $p_t(\mathbf{v} \mid \yy, \xt)$ involves the normalization constant $p_t(\yy \mid \xt)$, an intractable integral over $\mathbf{v}$. We approximate it by a density $q(\mathbf{v})$ that maximizes the free energy
\begin{equation}
\Fq(q) = \int q(\mathbf{v}) \log\frac{p_t(\mathbf{v}, \yy \mid \xt)}{q(\mathbf{v})}\, d\mathbf{v},
\label{eq:elbo}
\end{equation}
a lower bound of $\log p_t(\yy \mid \xt)$. This is equivalent to minimizing $\KL[q \,\|\, p_t(\cdot \mid \yy, \xt)]$ \cite{quinn2006variational, beal2003variational}.\\
We factorize the approximating distribution $q$ into three blocks:
\begin{equation}
q(\mathbf{v}) = q(\xz)\, q(\gamma_r)\, q(\gamma_b),
\label{eq:MF}
\end{equation}
with $q(\xz) = \Norm(\bm\mu, \bm\Sigma)$, $q(\gamma_r) = \Gam(\tilde a_r, \tilde b_r)$ and $q(\gamma_b) = \Gam(\tilde a_b, \tilde b_b)$, where $\Gam(x;a,b) = x^{a-1}\frac{b^ae^{-bx}}{\Gamma(a)}$.
\begin{proposition}[Conjugate variational inference]
\label{prop:vb}
Under \eqref{eq:M1}--\eqref{eq:M4} and the factorization \eqref{eq:MF}, the fixed point of the updates is
\begin{align}
\bm\Sigma &= \left(\langle\gamma_b\rangle\, \Ab^\top \Ab + \langle\gamma_r\rangle\, \Iden\right)^{-1},
\label{eq:Sigma_full}\\
\bm\mu &= \bm\Sigma\Bigl(\langle\gamma_b\rangle\, \Ab^\top \yy + \langle\gamma_r\rangle\, \xhz(\xt)\Bigr),
\label{eq:mu_full}
\end{align}
that is \eqref{eq:sigma_post}--\eqref{eq:mu_post} where the unknown precisions are replaced by their posterior expectations.
\end{proposition}
Coordinate Ascent Variational Inference \cite{beal2003variational} maximizes \eqref{eq:elbo} by updating each block in turn, the others held fixed:
\begin{equation}
q_j^r(v_j) \propto \exp\!\left[\bigl\langle \log p_t(\mathbf{v}, \yy \mid \xt)\bigr\rangle_{\prod_{l \ne j} q_l}\right].
\label{eq:qref}
\end{equation}
The structure of \eqref{eq:logjoint} guarantees by conjugacy that each block stays in its family. The expectation of \eqref{eq:logjoint} under $q(\gamma_r)\, q(\gamma_b)$ is quadratic in $\xz$, hence \eqref{eq:Sigma_full}--\eqref{eq:mu_full}. The expectation under $q(\xz)$, with $\langle\|\yy - \Ab\xz\|^2\rangle = \|\yy - \Ab\bm\mu\|_2^2 + \mathrm{tr}(\Ab^\top \Ab\, \bm\Sigma)$, gives
\begin{align}
\tilde a_b &= \tfrac{m}{2},
\label{eq:atilde_b}\\
\tilde b_b &=  \tfrac12\!\left(\|\yy - \Ab\bm\mu\|^2_2 + \mathrm{tr}(\Ab^\top \Ab\, \bm\Sigma)\right),
\label{eq:btilde_b}
\end{align}
and similarly
\begin{align}
\tilde a_r &= \tfrac{n}{2},
\label{eq:atilde_r}\\
\tilde b_r &= \tfrac12\!\left(\|\bm\mu - \xhz\|^2_2 + \mathrm{tr}(\bm\Sigma)\right).
\label{eq:btilde_r}
\end{align}
For $X \sim \Gam(a, b)$, $\Exp[X] = a/b$, hence
\begin{equation}
\mathbb{E}_q[\gamma_b]=\langle \gamma_b \rangle = \frac{\tilde a_b}{\tilde b_b}, \qquad \mathbb{E}_q[\gamma_r]=\langle \gamma_r \rangle = \frac{\tilde a_r}{\tilde b_r}.
\label{eq:exp_gamma}
\end{equation}
Alternating these three updates increases $\Fq$ at each step \cite{beal2003variational}.
\subsection{Separable factorization and fast updates}
\label{sec:methode:separable}
Update \eqref{eq:Sigma_full} requires inverting an $n \times n$ matrix, with $n$ on the order of $2\times10^5$ for a $256\times256$ color image, at each iteration of each reverse step. This cost is prohibitive. We therefore factorize $q(\xz)$ into independent components, as in variational super-resolution \cite{babacan2011sr}:
\begin{equation}
q(\mathbf{v}) = \left(\prod_{i=1}^{n} q_i(x_{0,i})\right) q(\gamma_r)\, q(\gamma_b),
\quad q_i(x_{0,i}) = \Norm(\mu_i, \sigma_i^2).
\label{eq:MFsep}
\end{equation}
\begin{corollary}[Separable updates]
\label{cor:separable}
Under \eqref{eq:MFsep}, the covariance of $q(\xz)$ is diagonal. Collecting the per-component variances in the vector $\bm\sigma^2 = (\sigma_1^2, \dots, \sigma_n^2)$, we have $\bm\Sigma = \mathrm{diag}(\bm\sigma^2)$. The per-component mean-field target is Gaussian, $q_i^r = \Norm(\mu_i^r, (\sigma_i^r)^2)$, and reads \cite{zheng2015efficient}
\begin{align}
&((\sigma_i^r)^2)^{-1} = \langle \gamma_b\rangle\,(\Ab^\top \Ab)_{ii} + \langle \gamma_r\rangle,
\label{eq:Sigmar_i}\\
&\mu_i^r = (\sigma_i^r)^2\Bigl(\langle \gamma_b\rangle\bigl[(\Ab^\top \yy)_i - (\Ab^\top \Ab\,\bm\mu_{-i})_i\bigr] + \langle \gamma_r\rangle\,(\xhz)_i\Bigr),
\label{eq:mur_i}
\end{align}
where $\bm\mu_{-i}$ is $\bm\mu$ with the $i$-th component set to zero. \\Coordinate Ascent Variational Inference would treat the $n$ components one by one, a sequential sweep impractical at this dimension. We adopt the scheme of Fraysse \textit{et al.} \cite{fraysse2014measure, zheng2015efficient}, which updates them simultaneously, hence \eqref{eq:Sigmar_i}--\eqref{eq:mur_i}. Writing $\bm\mu^{(k)}$ and $\bm\sigma^{2,(k)}$ for the parameters of $q(\xz)$ at inner iteration $k$, and $\bm\mu^{\mathrm{r}}$, $\bm\sigma^{2,\mathrm{r}} = ((\sigma_1^r)^2,\dots,(\sigma_n^r)^2)$ for the mean-field target \eqref{eq:Sigmar_i}--\eqref{eq:mur_i}, the update is a relaxed step
\begin{align}
\bm\mu^{(k+1)} &= (1-s_k)\,\bm\mu^{(k)} + s_k\,\bm\mu^{\mathrm{r}},
\label{eq:opt_step_mu}\\
\bm\sigma^{2,(k+1)} &= (1-s_k)\,\bm\sigma^{2,(k)} + s_k\,\bm\sigma^{2,\mathrm{r}},
\label{eq:opt_step_sigma}
\end{align}
where the step size $s_k$ maximizes the free energy $\mathcal{F}$ and admits a closed form \cite{fraysse2014measure}. After $K$ iterations, $\bm\mu_{\mathrm{post}}
= \bm\mu^{(K)}$.
\end{corollary}
By conjugacy, the expectation of \eqref{eq:logjoint} under $q(\xz)$ keeps
$q(\gamma_b)$ and $q(\gamma_r)$ in the Gamma family, with parameters
\begin{align}
\tilde a_b &= \frac{m}{2}, \label{eq:atilde_b_sep}\\
\tilde b_b &= \frac12 \Bigl( \|\yy - \Ab\bm\mu\|_2^2 + \sum_{i=1}^{n} (\Ab^\top\Ab)_{ii}\,\sigma_i^2 \Bigr), \label{eq:btilde_b_sep}
\end{align}
and similarly
\begin{align}
\tilde a_r &= \frac{n}{2}, \label{eq:atilde_r_sep}\\
\tilde b_r &= \frac12 \Bigl( \|\bm\mu - \xhz\|_2^2 + \sum_{i=1}^{n} \sigma_i^2 \Bigr), \label{eq:btilde_r_sep}
\end{align}
with expectations $\langle \gamma_b\rangle = \tilde a_b/\tilde b_b$ and $\langle \gamma_r\rangle = \tilde a_r/\tilde b_r$. The relaxed update of $q(\xz)$ and the two Gamma updates are alternated within each reverse step; Algorithm~\ref{alg:FB-GDM} collects them.
\subsection{Algorithm}
\label{sec:methode:algo}
Algorithm~\ref{alg:FB-GDM} details one reverse step, organized like Algorithm~\ref{alg:pigdm}, without guidance weighting. All components not explicitly modified in Algorithm~\ref{alg:FB-GDM} are kept strictly identical to those of $\Pi$GDM. The initialization $\bm\Sigma \gets (1-\abart)\,\Iden$ reuses $\Pi$GDM's heuristic $r_t^2 = 1-\abart$: FB-GDM starts from $\Pi$GDM's setting, then refines it.

\begin{algorithm}[h]
\caption{FB-GDM, one reverse step in the DDPM sampler}
\label{alg:FB-GDM}
\begin{algorithmic}[1]
\REQUIRE $\xt,\ \yy,\ \Ab,\ \eps_\theta,\ \abart,\ K$
\STATE \textbf{$\triangleright$ Score computation}
\STATE $\eps_\theta \gets \eps_\theta(\xt, t)$ \COMMENT{network}
\STATE $\xhz \gets \big(\xt - \sqrt{1-\abart}\,\eps_\theta\big)\big/\sqrt{\abart}$ \COMMENT{Tweedie, \eqref{eq:tweedie}}
\STATE $\bm\mu \gets \xhz,\qquad \bm\Sigma \gets (1-\abart)\,\Iden$ \COMMENT{initialization}
\STATE $q(\gamma_b) \gets \mathcal{G}(\tilde a_b,\tilde b_b)$, $\ \langle\gamma_b\rangle \gets \tilde a_b/\tilde b_b$ \COMMENT{\eqref{eq:atilde_b_sep},\eqref{eq:btilde_b_sep},\eqref{eq:exp_gamma}}
\STATE $q(\gamma_r) \gets \mathcal{G}(\tilde a_r,\tilde b_r)$, $\ \langle\gamma_r\rangle \gets \tilde a_r/\tilde b_r$ \COMMENT{\eqref{eq:atilde_r_sep},\eqref{eq:btilde_r_sep},\eqref{eq:exp_gamma}}
\FOR{$k = 1,\ldots,K$}
  \STATE $q^r(\xz)=\Norm(\bm\mu^r,\bm\Sigma^r)$ \COMMENT{\eqref{eq:Sigmar_i},\eqref{eq:mur_i}}
  \STATE $q(\xz)=\Norm(\bm\mu,\bm\Sigma)$
  \COMMENT{\eqref{eq:opt_step_mu},\eqref{eq:opt_step_sigma}}
  \STATE $q(\gamma_b) \gets \mathcal{G}(\tilde a_b,\tilde b_b)$, $\langle\gamma_b\rangle \gets \tilde a_b/\tilde b_b$ \COMMENT{\eqref{eq:atilde_b_sep},\eqref{eq:btilde_b_sep}}
  \STATE $q(\gamma_r) \gets \mathcal{G}(\tilde a_r,\tilde b_r)$, $\ \langle\gamma_r\rangle \gets \tilde a_r/\tilde b_r$ \COMMENT{\eqref{eq:atilde_r_sep},\eqref{eq:btilde_r_sep}}
\ENDFOR
\STATE $\bm\mu_{\text{post}} \gets \bm\mu$
\STATE $\text{score}(\yy \mid \xt) \gets \big(\partial\xhz/\partial\xt\big)^{\!\top}\big(\bm\mu_{\text{post}} - \xhz\big)/\big(1-\abart\big)$
\STATE $\text{score}(\xt)\gets-\,\eps_\theta/\sqrt{1-\abart}$
\COMMENT{\eqref{eq:score}}
\STATE $\text{score}(\xt \mid \yy) \gets \text{score}(\xt) + \text{score}(\yy\mid \xt)$ \COMMENT{\eqref{eq:score_closed}}
\STATE \textbf{$\triangleright$ DDPM reverse step}
\STATE $\boldsymbol{z} \sim \Norm(0,\Iden)$
\STATE $\bm x_{t-1} \gets \dfrac{1}{\sqrt{\alpha_t}}\Big(\xt + \beta_t\,\text{score}(\xt \mid \yy)\Big) + \sqrt{\tilde\beta_t}\,\boldsymbol{z}$ \COMMENT{\eqref{eq:ddpm_score_conditionnel}}
\RETURN $\bm x_{t-1}$
\end{algorithmic}
\vspace{-3pt}
\end{algorithm}
\vspace{-3pt}
 % ===========================================================
 % ===========================================================
\section{Experiments}
\label{sec:experiments}
\begin{figure*}[t]
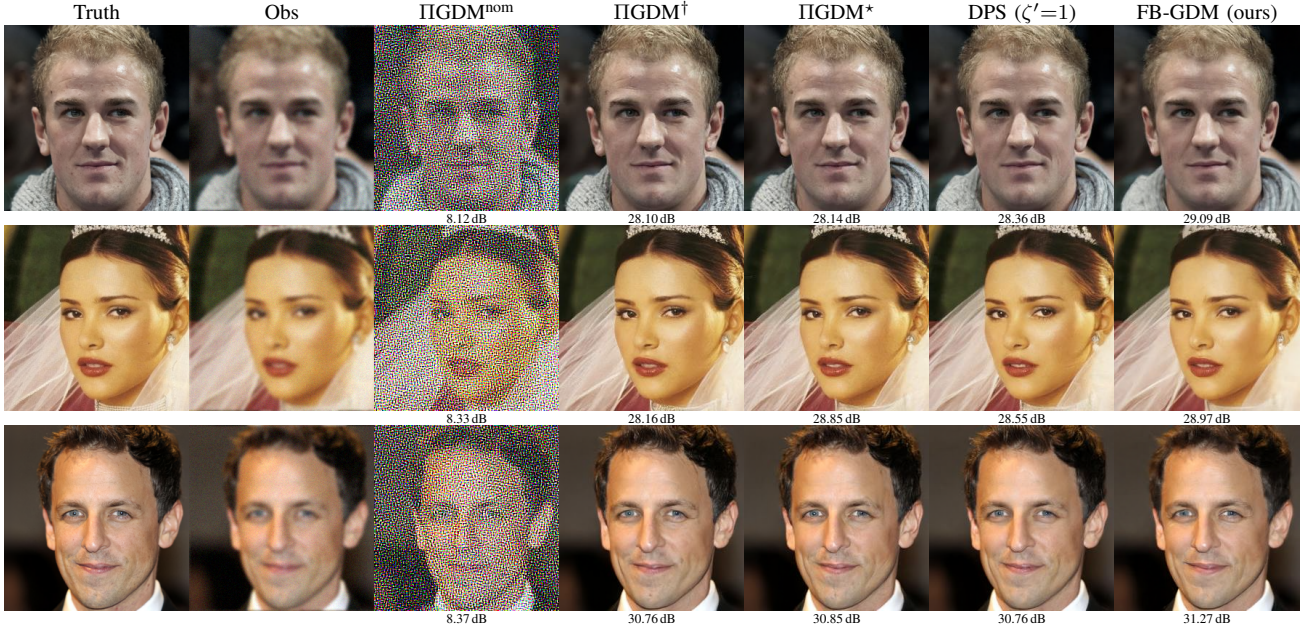

\centering
\cmpsetup{0.135}

\begin{tabular}{@{}ccccccc@{}}
\footnotesize Truth &
\footnotesize Obs &
\footnotesize $\Pi$GDM$^{\text{nom}}$ &
\footnotesize $\Pi$GDM$^{\dag}$ &
\footnotesize $\Pi$GDM$^{\star}$ &
\footnotesize DPS ($\zeta'{=}1$) &
\footnotesize FB-GDM (ours) \\[1pt]

\cmprowSeven{blur/img_8}{}{8.12\,dB}{28.10\,dB}{28.14\,dB}{28.36\,dB}{29.09\,dB}
\cmprowSeven{blur/img_11}{}{8.33\,dB}{28.16\,dB}{28.85\,dB}{28.55\,dB}{28.97\,dB}
\cmprowSeven{blur/img_25}{}{8.37\,dB}{30.76\,dB}{30.85\,dB}{30.76\,dB}{31.27\,dB}
 
\end{tabular}
\caption{Visual results for Gaussian deblurring ($9\times9$, $\text{FWHM}= 3.5$ px) at fixed measurement $\mathrm{SNR}{=}20$ dB. From left to right: ground truth, degraded observation, nominal $\Pi$GDM, one-dimensional $\Pi$GDM oracle, two-dimensional $\Pi$GDM oracle, DPS, and FB-GDM. PSNR values are reported below each reconstruction.}
\label{fig:blur}
\vspace{-10pt}
\end{figure*}
\subsection{Protocol}
\subsubsection{Setup}
\label{sec:setup}
The prior is an unconditional DDPM~\cite{ho2020ddpm} pretrained on CelebA-HQ at $256\times256$~\cite{karras2018progressive}\footnote{Model: \url{https://huggingface.co/google/ddpm-celebahq-256}}\footnote{Data: \url{https://www.kaggle.com/datasets/denislukovnikov/celebahq256-images-only}} and never fine-tuned. The sampler is the ancestral DDPM with $T{=}1000$ steps. Except for the noise-robustness study, the measurement signal-to-noise ratio is fixed at $\mathrm{SNR}=20$~dB.
FB-GDM runs with $K{=}100$ inner iterations of the variational inference at each reverse step. It receives only the observation $\yy$ and the operator $\Ab$: neither the noise level nor the ground truth, and no hyperparameter is adjusted from one experiment to the next.\\
\indent We measure reconstruction quality with PSNR and SSIM~\cite{wang2004ssim}.
\subsubsection{Comparison with the state of the art}
We compare FB-GDM (Algorithm~\ref{alg:FB-GDM}) with $\Pi$GDM  (Algorithm~\ref{alg:pigdm}) ~\cite{song2023pigdm} and DPS~\cite{chung2023dps}. Both are run in the setup of Section~\ref{sec:setup}: same pretrained prior, same DDPM sampler, same degraded observations and the same noise seed for the initialization $\bm x_T$ of the reverse trajectory. Only the conditional score changes from one method to the next. We run $\Pi$GDM under three settings that use more and more information about the problem. 
$\Pi$GDM$^{\text{nom}}$ is the nominal setting: $w=1$ and $\sigma_b^2$ fixed at the simulated value, with no search.
$\Pi$GDM$^{\dag}$ searches only the guidance weight $w$ over $\{0.01,0.05,0.1,0.5,1,2\}$ and keeps $\sigma_b^2$ fixed at the simulated value. Both settings therefore consider the noise level known.
$\Pi$GDM$^{\star}$ is the oracle: it jointly searches $w$ and the value of $\sigma_b^2$ used by the guidance on the 2D grid $\{0.01,0.05,0.1,0.5,1,2\}\times\{1,2.5,5,10\}{\times}10^{-3}$, so that the guidance variance is no longer tied to the simulated value.
For $\Pi$GDM$^{\dag}$ and $\Pi$GDM$^{\star}$, the candidates are ranked on PSNR against the ground truth.
We run DPS with its single hyperparameter $\zeta'$ set to the values recommended per operator in \cite{chung2023dps}.
 % ===========================================================
 % ===========================================================
\subsection{First experiment: Gaussian deblurring}
\label{subsec:exp_blur}
The operator is a $9\times9$ Gaussian blur with a full width at half maximum (FWHM) of $3.5$ pixels, and the measurement signal-to-noise ratio is fixed at $\mathrm{SNR} = 20$ dB. We evaluate on $N_{\text{img}} = 30 $ CelebA-HQ validation images, with the same noise seed for the initialization $x_T$ for every method so that all reconstructions see the same observation. \\
\indent Table~\ref{tab:blur} shows that FB-GDM achieves the highest mean PSNR and SSIM, with $28.86$~dB and $0.822$, respectively, without using the noise level, the ground truth, or task-specific calibration. Figure~\ref{fig:blur} shows clean reconstructions that remain visually faithful to the ground truth.\\
\indent At its nominal setting, $\Pi$GDM$^{\text{nom}}$ degrades to $15.28$~dB despite being given the true noise level, because the unit weight amplifies noise along the near-singular directions of the blur. Searching over $w$ raises $\Pi$GDM$^{\dag}$ to $28.75$~dB, while the additional search over the effective guidance variance brings only a marginal gain, with $\Pi$GDM$^{\star}$ reaching $28.85$~dB. These settings require six and twenty-four generations, respectively, and use the ground truth for selection. DPS reaches $28.79$~dB at $\zeta'{=}1$. Visually, FB-GDM, DPS, and $\Pi$GDM$^{\star}$ are nearly indistinguishable in Figure~\ref{fig:blur}, whereas $\Pi$GDM$^{\text{nom}}$ is dominated by noise.
\vspace{-13pt}
\begin{table}[h]
\centering
\caption{Gaussian deblurring ($9\times9$, $\text{FWHM}= 3.5$ px) on CelebA-HQ at fixed measurement $\mathrm{SNR}=20$ dB. PSNR and SSIM are averaged over $N_{\text{img}}=30$ validation images.}
\label{tab:blur}
\setlength{\tabcolsep}{11pt}
\renewcommand{\arraystretch}{0.5}
\begin{tabular}{lcc|ccc}
\toprule
Method & \multicolumn{2}{c}{PSNR} & \multicolumn{2}{c}{SSIM} \\
\cmidrule(lr){2-3}\cmidrule(lr){4-5}
 & Mean & Std & Mean & Std \\
\midrule
FB-GDM (ours) & \textbf{28.86} & 2.29 & \textbf{0.822} & 0.059  \\
\midrule
DPS ($\zeta'=1$) & 28.79 & 2.31 & 0.818 & 0.060 \\
\midrule
$\Pi$GDM$^{\mathrm{nom}}$ & 15.28 & 6.48 & 0.342 & 0.313  \\
$\Pi$GDM$^{\dag}$  & 28.75 & 2.33 & 0.819 & 0.059  \\
$\Pi$GDM$^{\star}$  & 28.85 & 2.23 & 0.818 & 0.057 \\
\bottomrule
\end{tabular}
\vspace{-20pt}
\end{table}
 % ===========================================================
 % ===========================================================
\subsection{Second experiment: robustness to the operator}
\label{subsec:exp_operator}
 \begin{figure*}[t]
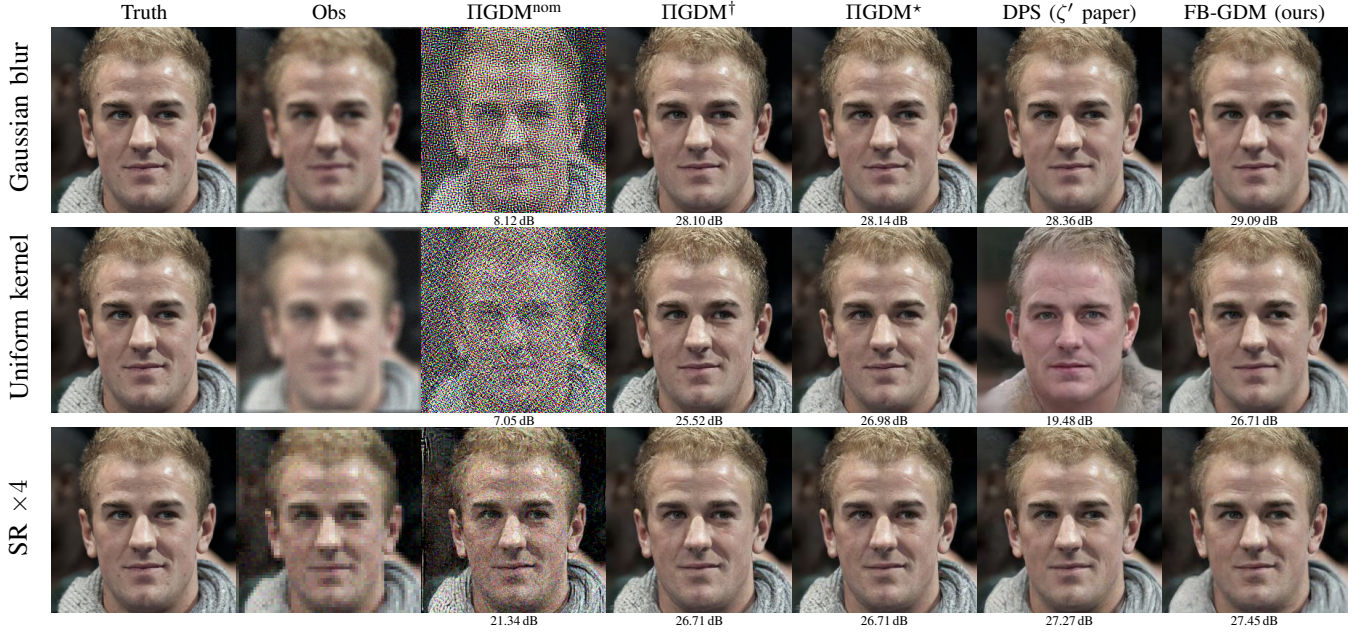

\centering
\cmpsetup{0.135}
\begin{tabular}{c@{\hspace{3mm}}ccccccc@{}}
&
\footnotesize Truth &
\footnotesize Obs &
\footnotesize $\Pi$GDM$^{\text{nom}}$ &
\footnotesize $\Pi$GDM$^{\dag}$ &
\footnotesize $\Pi$GDM$^{\star}$ &
\footnotesize DPS ($\zeta'$ paper) &
\footnotesize FB-GDM (ours) \\[2pt]
\cmprowSevenOp
  {Gaussian blur}
  {blur/img_8}
  {}
  {8.12\,dB}
  {28.10\,dB}
  {28.14\,dB}
  {28.36\,dB}
  {29.09\,dB}

\cmprowSevenOp
  {Uniform kernel}
  {avg1/img_8}
  {}
  {7.05\,dB}
  {25.52\,dB}
  {26.98\,dB}
  {19.48\,dB}
  {26.71\,dB}

\cmprowSevenOp
  {SR $\times4$}
  {sr/img_8}
  {}
  {21.34\,dB}
  {26.71\,dB}
  {26.71\,dB}
  {27.27\,dB}
  {27.45\,dB}
\end{tabular}
\caption{Visual results for operator robustness at fixed measurement $\mathrm{SNR}{=}20$ dB. Each row corresponds to a different forward operator: Gaussian blur, uniform kernel and super-resolution $\times4$. From left to right: ground truth, degraded observation, nominal $\Pi$GDM, one-dimensional $\Pi$GDM oracle, two-dimensional $\Pi$GDM oracle, DPS, and FB-GDM. PSNR values are reported below each reconstruction.}
\label{fig:operator}
\vspace{-10pt}
\end{figure*}
We now leave FB-GDM untouched and change only the forward operator $\Ab$ of \eqref{eq:inverse_pb}. We test three operators at the same measurement SNR of $20$ dB. The first is the Gaussian blur of the first experiment. The second is a $13\times13$ uniform square kernel with constant coefficients. The third is super-resolution $\times4$: a $9\times9$ Gaussian anti-aliasing kernel with a full width at half maximum (FWHM) of $3.5$ pixels, followed by a decimation of factor $4$. All operators are implemented as circular convolutions.\\
\indent Tables~\ref{tab:operator_psnr} and \ref{tab:operator_ssim}, and Figure~\ref{fig:operator} show that FB-GDM is robust to the change of operator: it stays between $26.76$ and $28.86$~dB, with no failure case, without access to the ground truth and without any operator-specific retuning.\\
\begin{table}[h]
\centering
\caption{Operator robustness at fixed measurement $\mathrm{SNR}=20$ dB. PSNR values are reported as mean and standard deviation over the test images. Bold indicates the best method for each operator.}
\label{tab:operator_psnr}
\setlength{\tabcolsep}{8.38pt}
\renewcommand{\arraystretch}{1.1}
\scriptsize
\begin{tabular}{lcc|cc|cc}
\toprule
\textbf{PSNR}
& \multicolumn{2}{c|}{Gaussian}
& \multicolumn{2}{c|}{Uniform kernel}
& \multicolumn{2}{c}{SR $\times4$}\\
\cmidrule(lr){2-3}\cmidrule(lr){4-5}\cmidrule(lr){6-7}
Method & Mean & Std & Mean & Std & Mean & Std\\
\midrule
FB-GDM (ours)
& \textbf{28.86} & 2.29
& 26.76 & 2.06
& 27.23 & 2.09 \\
\midrule
DPS
& 28.79 & 2.31
& 19.53 & 1.53
& \textbf{27.82} & 2.16 \\
\midrule
$\Pi$GDM$^{\mathrm{nom}}$
& 15.28 & 6.48
& 17.02 & 7.65
& 20.53 & 1.36\\
$\Pi$GDM$^{\dag}$
& 28.75 & 2.33
& 26.83 & 2.16
& 27.58 & 2.07 \\
$\Pi$GDM$^{\star}$
& 28.85 & 2.23
& \textbf{27.48} & 2.02
& 27.58 & 2.07\\
\bottomrule
\end{tabular}
\vspace{-10pt}
\end{table}
\begin{table}[h]
\centering
\caption{Operator robustness at fixed measurement $\mathrm{SNR}=20$ dB. SSIM values are reported as mean and standard deviation over the test images. Bold indicates the best method for each operator.}
\label{tab:operator_ssim}
\setlength{\tabcolsep}{7pt}
\renewcommand{\arraystretch}{1.1}
\scriptsize
\begin{tabular}{lcc|cc|cc}
\toprule
\textbf{SSIM}
& \multicolumn{2}{c|}{Gaussian}
& \multicolumn{2}{c|}{Uniform kernel}
& \multicolumn{2}{c}{SR $\times4$}\\
\cmidrule(lr){2-3}\cmidrule(lr){4-5}\cmidrule(lr){6-7}
Method & Mean & Std & Mean & Std & Mean & Std \\
\midrule
FB-GDM (ours)
& \textbf{0.822} & 0.059
& 0.752 & 0.070
& 0.769 & 0.065\\
\midrule
DPS
& 0.818 & 0.060
& 0.532 & 0.097
& \textbf{0.787} & 0.063 \\
\midrule
$\Pi$GDM$^{\mathrm{nom}}$
& 0.342 & 0.313
& 0.419 & 0.313
& 0.417 & 0.079\\
$\Pi$GDM$^{\dag}$
& 0.819 & 0.059
& 0.759 & 0.067
& 0.780 & 0.062\\
$\Pi$GDM$^{\star}$
& 0.818 & 0.057
& \textbf{0.769} & 0.065
& 0.780 & 0.062\\
\bottomrule
\end{tabular}
\end{table}
\indent For the Gaussian blur, FB-GDM attains the best PSNR of the comparison, within $0.01$~dB of $\Pi$GDM$^{\star}$, without tuning. DPS achieves a comparable PSNR at its recommended scale. In contrast, $\Pi$GDM$^{\text{nom}}$ degrades to $15.28$~dB, even though it is given the true noise level, because its fixed unit weight is not suited to the operator.\\
\indent For super-resolution $\times 4$, FB-GDM stays within $0.4$~dB of both oracles $\Pi$GDM$^{\dag}$ and $\Pi$GDM$^{\star}$ without tuning, while the tuned DPS attains the best scores ($27.82$~dB, $0.787$ SSIM) and $\Pi$GDM$^{\text{nom}}$ reaches only $20.53$~dB. This again shows that the nominal guidance scale is not robust across operators.\\
\indent On the uniform kernel, FB-GDM reaches $26.76$~dB and $0.752$ SSIM without tuning, close to $\Pi$GDM$^{\dag}$ ($26.83$~dB, $0.759$ SSIM). The fully tuned oracle $\Pi$GDM$^{\star}$ performs best ($27.48$~dB, $0.769$ SSIM), but the gain over FB-GDM remains moderate: $0.72$~dB and $0.017$ SSIM. By contrast, fixed-scale methods degrade strongly: $\Pi$GDM$^{\mathrm{nom}}$ falls to $17.02$~dB and $0.419$ SSIM, while DPS reconstruction drops to $19.53$~dB and $0.532$ SSIM and visually replaces the observed identity by a hallucinated face unsupported by the measurements.\\
\indent The gain of $\Pi$GDM$^{\star}$ comes from a costly calibration: for this operator, both the guidance weight and the effective guidance variance must be re-adjusted, with the selected variance differing from the true noise variance by about a factor of three. Finding this setting required an enlarged two-dimensional ground-truth grid with more than $130$ evaluations. FB-GDM therefore remains close to the best per-operator oracle from a single unsupervised pass, without access to the noise level or the ground truth.
 % ===========================================================
 % ===========================================================
\subsection{Third experiment: robustness to the noise level}
\label{subsec:exp_noise}
\indent We now fix the operator, here super-resolution $\times4$, and vary the signal-to-noise ratio over a wide range, $15$ values linearly spaced from 1 to 30 dB. FB-GDM is again left untouched: it re-infers its precisions at every reverse step, without the true noise as input.\\
\indent To measure how much tuning $\Pi$GDM needs when the noise changes, we use frozen oracles. We calibrate $\Pi$GDM$^{\star}$, the full 2D grid over $w$ and $\sigma_b^2$, at two reference SNR levels, $\mathrm{SNR} \in \{5,30\}$ dB. We then freeze each best setting and sweep the whole noise range without re-tuning it. This isolates one effect: the price of a fixed hyperparameter when the true observation noise drifts away from the level it was tuned for.\\
\indent Figure~\ref{fig:noise} shows that FB-GDM is robust to the noise level: it degrades smoothly as the SNR decreases, as expected when the measurement carries less information, and stays within $0.7$~dB of the best frozen oracle at every SNR, without the true noise variance as input and re-inferring its precisions at each reverse step.\\
\indent The two frozen $\Pi$GDM$^{\star}$ curves show the two failure modes of a fixed calibration. The blue curve, calibrated at $30$ dB, weights the data strongly: it slightly exceeds FB-GDM at high $\mathrm{SNR}$, but collapses below $13$ dB at $\mathrm{SNR}=1$ dB, more than $9$ dB under FB-GDM. The green curve, calibrated at $5$ dB, makes the opposite trade-off: it leads marginally at low $\mathrm{SNR}$, but saturates near $25.5$ dB at high $\mathrm{SNR}$, about $3$ dB below FB-GDM. Each frozen curve crosses FB-GDM once and lies above it only in the SNR regime for which that curve was tuned. FB-GDM follows the upper envelope of both, recovering by inference what each oracle reaches only through a ground-truth search at one noise level.
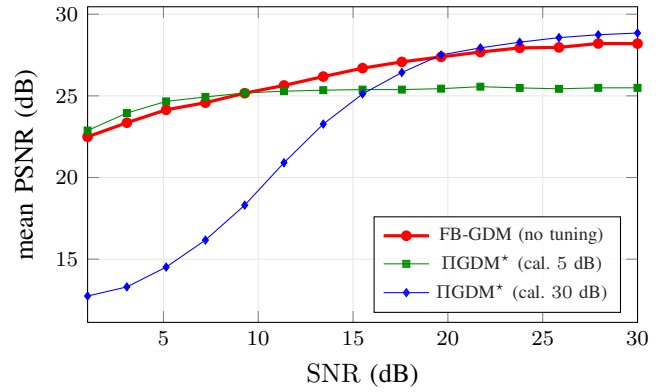
\begin{figure}[h]
\centering
\begin{tikzpicture}
\begin{axis}[
  width=\columnwidth, height=0.65\columnwidth,
  xlabel={$\mathrm{SNR}$ (dB)}, ylabel={mean PSNR (dB)},
  xmin=1, xmax=30,
  legend pos=south east, legend style={font=\scriptsize, fill opacity=0.85},
  grid=both, grid style={gray!18}, tick label style={font=\footnotesize},
]
\addplot[red, very thick, mark=*, mark size=1.4] coordinates {
(1.000,22.493) (3.071,23.356) (5.143,24.146) (7.214,24.581) (9.286,25.164) (11.357,25.639) (13.429,26.182) (15.500,26.697) (17.571,27.080) (19.643,27.388) (21.714,27.684) (23.786,27.939) (25.857,27.968) (27.929,28.207) (30.000,28.199)
};
\addlegendentry{FB-GDM (no tuning)}
 
\addplot[green!55!black, mark=square*, mark size=1.3] coordinates {
(1.000,22.877) (3.071,23.947) (5.143,24.662) (7.214,24.934) (9.286,25.179) (11.357,25.287) (13.429,25.348) (15.500,25.384) (17.571,25.384) (19.643,25.442) (21.714,25.563) (23.786,25.488) (25.857,25.433) (27.929,25.492) (30.000,25.494)
};
\addlegendentry{$\Pi$GDM$^{\star}$ (cal. $5$ dB)}
 
\addplot[blue, mark=diamond*, mark size=1.5] coordinates {
(1.000,12.741) (3.071,13.302) (5.143,14.514) (7.214,16.169) (9.286,18.308) (11.357,20.906) (13.429,23.273) (15.500,25.119) (17.571,26.431) (19.643,27.507) (21.714,27.941) (23.786,28.286) (25.857,28.567) (27.929,28.743) (30.000,28.846)
};
\addlegendentry{$\Pi$GDM$^{\star}$ (cal. $30$ dB)}
\end{axis}
\end{tikzpicture}
\caption{Robustness to the signal-to-noise ratio on super-resolution $\times4$, $\mathrm{SNR}$ swept from $1$ to $30$ dB. FB-GDM uses no tuning; each $\Pi$GDM$^{\star}$ curve is calibrated on the ground truth at a single SNR, then frozen.}
\label{fig:noise}
\vspace{-10pt}
\end{figure}

\subsection{Fourth experiment: out-of-distribution generalization and hallucinations}
\label{subsec:exp_ood}

We now change the image distribution itself. The operator is a $9\times9$ uniform square kernel with constant coefficients at $\mathrm{SNR}{=}20$~dB, and the prior remains the CelebA-HQ face model; we apply them to $N_{\mathrm{img}}{=}30$ ImageNet images\footnote{Data: \url{https://www.kaggle.com/datasets/dimensi0n/imagenet-256}.}, which the prior never saw. The averaging kernel suppresses most of the high-frequency content, increasing the influence of the prior and making hallucinations a potential failure mode out of distribution. As before, FB-GDM is left unchanged and still receives neither the noise level nor the ground truth.

Figure~\ref{fig:ood} shows that FB-GDM is robust to this change of image distribution. Its reconstructions remain consistent with the degraded observations and preserve the main structures of the non-face images, without introducing face-like artifacts. In contrast, DPS exhibits hallucinated structures that are not supported by the measurements: facial features appear on the cat image, while the roof and surrounding textures of the cabin are altered into face-like patterns. The transferred $\Pi$GDM$^{\star}$ calibration also produces visually consistent reconstructions, but relies on a calibration obtained in distribution, whereas FB-GDM requires no calibration transfer or retuning.
\begin{figure*}[t]
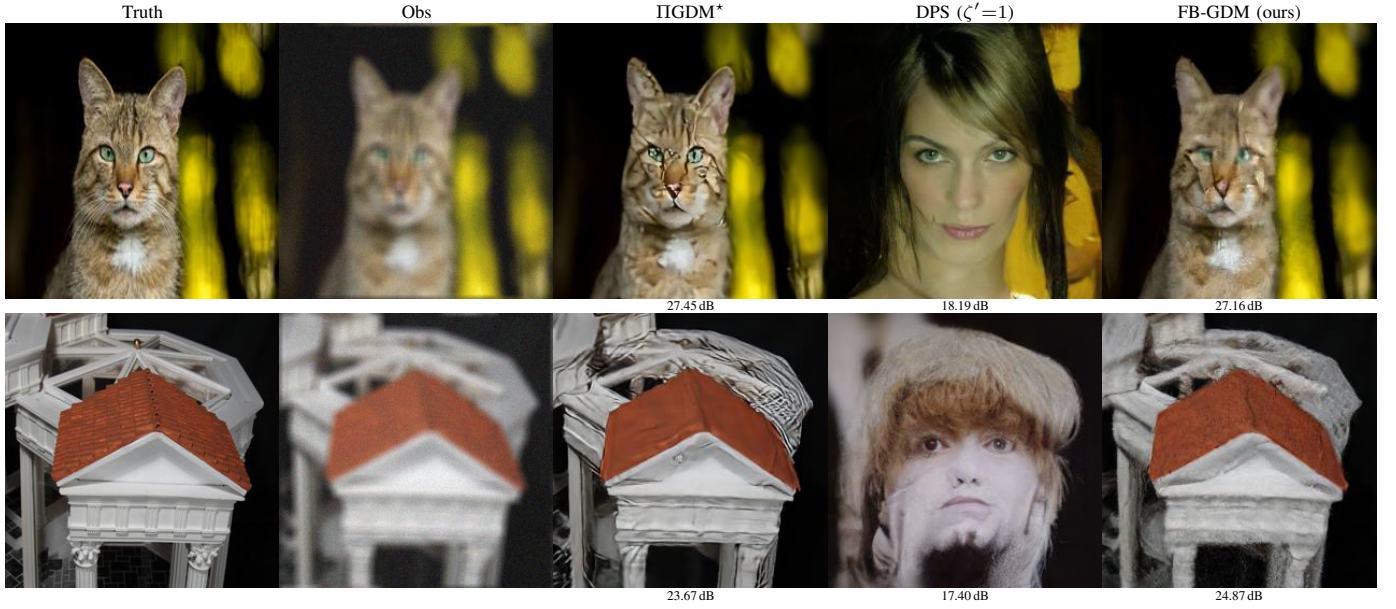

\centering
\setlength{\tabcolsep}{1pt}
\cmpsetup{0.2}
\begin{tabular}{@{}ccccc@{}}
\scriptsize Truth &
\scriptsize Obs &
\scriptsize $\Pi$GDM$^{\star}$ &
\scriptsize DPS ($\zeta'{=}1$) &
\scriptsize FB-GDM (ours) \\[1pt]
\cmprowFive{ood/ex2}{}{27.45\,dB}{18.19\,dB}{27.16\,dB}
\cmprowFive{ood/ex3}{}{23.67\,dB}{17.40\,dB}{24.87\,dB}
\end{tabular}
\caption{Visual results for out-of-distribution robustness. A CelebA-HQ face prior is applied to ImageNet images using a $9\times9$ uniform averaging kernel at $\mathrm{SNR}{=}20$~dB. From left to right: ground truth, observation, transferred $\Pi$GDM$^{\star}$, DPS ($\zeta'{=}1$), and FB-GDM. DPS hallucinates face-like structures that are unsupported by the observations, whereas FB-GDM remains faithful to the measurements while preserving the scene content.}
\label{fig:ood}
\vspace{-10pt}
\end{figure*}

\subsection{Computational cost}
\label{subsec:exp_cost}
We measure the cost of the three methods at a fixed operator, here super-resolution $\times 4$, on a single GeForce RTX 3070 at $256\times256$, with $T{=}1000$ DDPM steps for all methods.
Table~\ref{tab:cost} reports per-run times averaged over the
reconstructions of the noise-robustness study. One FB-GDM
reconstruction takes $165$s, against $122$s for $\Pi$GDM and $116$s for DPS. The $35\%$ overhead per run is the price of the $K{=}100$ variational updates. Each update is linear in $n$, so this cost is stable across operators, whereas that of $\Pi$GDM is driven by the conjugate-gradient solver, hence by the conditioning of $\Ab$.
The relevant comparison is the total cost of a calibrated result. A single FB-GDM pass suffices: under $3$ minutes. $\Pi$GDM$^{\dag}$ and $\Pi$GDM$^{\star}$ require $6$ and $24$ generations per calibration, about $12$ and $49$ minutes here, repeated at every change of operator or noise level, and both rank their candidates on the ground truth. The single run of DPS does not reflect the per-task tuning of $\zeta'$ performed offline by its authors, and this fixed setting fails on the uniform kernel, as reported in the operator-robustness experiment. The Bayesian inference thus replaces the grid search at the cost of a single run.
\begin{table}[!h]
\centering
\caption{Computational cost on super-resolution $\times 4$ at
$256\times256$, $T{=}1000$ DDPM steps for all methods. Per-run times are averaged over the reconstructions of the noise-robustness study (Section~\ref{subsec:exp_noise}).}
\label{tab:cost}
\setlength{\tabcolsep}{11pt}
\renewcommand{\arraystretch}{0.5}
\begin{tabular}{lccc}
\toprule
Method & Time per run (s) & Runs & Total (min) \\
\midrule
FB-GDM (ours)            & $165 \pm 6$ & $1$  & $2.8$  \\
\midrule
DPS                      & $116 \pm 1$ & $1$  & $1.9$  \\
\midrule
$\Pi$GDM$^{\mathrm{nom}}$ & $122 \pm 1$ & $1$  & $2.0$  \\
$\Pi$GDM$^{\dag}$        & $122 \pm 1$ & $6$  & $12.2$ \\
$\Pi$GDM$^{\star}$       & $122 \pm 1$ & $24$ & $48.9$ \\
\bottomrule
\end{tabular}
\vspace{-10pt}
\end{table}
\section{Conclusion and discussion}
We introduced FB-GDM, a fully Bayesian guidance method for diffusion models applied to high-dimensional linear inverse problems. Starting from the Gaussian approximation of $\Pi$GDM, we derived a closed-form conditional score and replaced its tuned hyperparameters with the precisions of the denoising approximation and of the observation likelihood, inferred at each reverse step by variational inference. A separable factorization keeps every update linear in the image dimension, so the method requires only the observation and the forward operator and runs at a cost comparable to a single $\Pi$GDM pass.\\
\indent FB-GDM matches ground-truth $\Pi$GDM oracles in distribution and stays close to a per-problem oracle under changes of operator, of noise level, and of image distribution, where fixed calibrations degrade or hallucinate; the inferred precision adapts to each new observation without supervision.\\
\indent Overall, FB-GDM requires no task-specific tuning and has a computational cost comparable to a single $\Pi$GDM run, making it applicable to a broad range of linear inverse problems.\\
\indent Future work will investigate scientific imaging problems that admit a linear or locally linearized forward model, including accelerated MRI, X-ray computed tomography, and radio-interferometric or deconvolution-based reconstruction in astronomy, where ground-truth images are typically unavailable.
\section*{Acknowledgment}
The authors thank the SIEN Department of École Normale Supérieure Paris-Saclay for providing access to the NVIDIA RTX 3070 workstations used in the experiments, and Dominique Lesselier for his careful reading of the manuscript and valuable comments.
\appendices
\section{Proof of the closed-form conditional score}
\subsection{General conditional score}
\label{app:closed_form_score1}
We derive the conditional score $\nabla_{\xt}\log p(\xt\mid\yy)$ under the model assumptions. We denote $r_t^2 = 1-\abart$.
Working likelihood and denoising approximation. We use a Gaussian observation likelihood with working variance $\sigma_b^2$, together with the Gaussian denoising approximation:
\begin{align}
\yy\mid\xz &\sim \Norm(\Ab\xz,\,\sigma_b^2\Iden), \label{eq:app_lik}\\
\xz\mid\xt &\sim \Norm(\xhz(\xt),\,r_t^2\Iden). \label{eq:app_prior_cond}
\end{align}
The Markov chain $\xt\to\xz\to\yy$ of Figure~\ref{fig:graphical_model} gives the conditional independence $\yy\perp\xt\mid\xz$, that is $p(\yy\mid\xz,\xt)=p(\yy\mid\xz)$.
\subsubsection*{Two gradient identities}
For the prior distribution,
\begin{equation}
\nabla_{\xt}\,p(\xt) = p(\xt)\,\nabla_{\xt}\log p(\xt).
\label{eq:app_grad_prior}
\end{equation}
For \eqref{eq:app_prior_cond}, Gaussian in $\xz$, we have $\log p(\xz\mid\xt) = -\tfrac{1}{2r_t^2}\|\xz-\xhz(\xt)\|^2 + \text{const.}$, where the constant does not depend on $\xt$. Since $\nabla_{\xt}\|\xz-\xhz(\xt)\|^2 = -2\,\frac{\partial\xhz}{\partial\xt}^\top\big(\xz-\xhz(\xt)\big)$, it follows that
\begin{equation}
\nabla_{\xt}\log p(\xz\mid\xt) = \frac{1}{r_t^2}\,\frac{\partial\xhz}{\partial\xt}^\top\big(\xz-\xhz(\xt)\big),
\label{eq:app_grad_logcond}
\end{equation}
and therefore
\begin{equation}
\nabla_{\xt}\,p(\xz\mid\xt) = p(\xz\mid\xt)\,\frac{1}{r_t^2}\,\frac{\partial\xhz}{\partial\xt}^\top\big(\xz-\xhz(\xt)\big).
\label{eq:app_grad_cond}
\end{equation}
\subsubsection*{Joint distribution}
By the chain and the conditional independence,
\begin{align}
p(\xt,\yy) &= p(\xt)\underbrace{\int p(\yy\mid\xz)\,p(\xz\mid\xt)\,d\xz}_{=\,p(\yy\mid\xt)} \notag\\&= p(\xt)\,p(\yy\mid\xt).
\label{eq:app_joint}
\end{align}
\subsubsection*{Gradient of the joint distribution}
Only $p(\xz\mid\xt)$ and $p(\xt)$ depend on $\xt$, hence
\begin{align}
\nabla_{\xt}\,p(\xt,\yy) ={}& p(\xt)\int p(\yy\mid\xz)\,\nabla_{\xt}\,p(\xz\mid\xt)\,d\xz \notag\\
&+ \nabla_{\xt}\,p(\xt)\int p(\yy\mid\xz)\,p(\xz\mid\xt)\,d\xz.
\label{eq:app_two_terms}
\end{align}
The second integral equals $p(\yy\mid\xt)$; by \eqref{eq:app_grad_prior}, the second term equals $p(\xt)\,p(\yy\mid\xt)\,\nabla_{\xt}\log p(\xt)$. For the first, we inject \eqref{eq:app_grad_cond}:
\begin{equation}
\frac{p(\xt)}{r_t^2}\,\frac{\partial\xhz}{\partial\xt}^\top\!\int p(\yy\mid\xz)\,p(\xz\mid\xt)\,\big(\xz-\xhz\big)\,d\xz.
\label{eq:app_first_term}
\end{equation}
Bayes' rule at fixed $\xt$ gives $p(\yy\mid\xz)\,p(\xz\mid\xt) = p(\yy\mid\xt)\,p(\xz\mid\xt,\yy)$, hence
\begin{align}
&\int p(\yy\mid\xz)\,p(\xz\mid\xt)\,\big(\xz-\xhz\big)\,d\xz \notag\\
&\quad = p(\yy\mid\xt)\!\int p(\xz\mid\xt,\yy)\,\big(\xz-\xhz\big)\,d\xz \notag\\
&\quad = p(\yy\mid\xt)\,\big(\Exp_{\xz\mid\xt,\yy}[\xz]-\xhz\big).
\label{eq:app_bayes}
\end{align}
Collecting \eqref{eq:app_two_terms}--\eqref{eq:app_bayes},
\begin{align}
\nabla_{\xt}\,p(\xt,\yy) = p(\xt)\,p(\yy\mid\xt)\Big[&\tfrac{1}{r_t^2}\frac{\partial\xhz}{\partial\xt}^\top\big(\Exp_{\xz\mid\xt,\yy}[\xz]-\xhz\big) \notag\\
&+ \nabla_{\xt}\log p(\xt)\Big].
\label{eq:app_grad_joint_final}
\end{align}
\subsubsection*{Conclusion}
Since $p(\yy)$ does not depend on $\xt$, we have $\nabla_{\xt}\log p(\xt\mid\yy)=\nabla_{\xt}\log p(\xt,\yy)=\nabla_{\xt}p(\xt,\yy)/p(\xt,\yy)$. Dividing \eqref{eq:app_grad_joint_final} by \eqref{eq:app_joint}:
\begin{equation}
\begin{aligned}
\nabla_{\xt}\log p(\xt\mid\yy) ={}& \frac{1}{r_t^2}\Big(\frac{\partial\xhz}{\partial\xt}\Big)^{\!\top}\big(\Exp_{\xz\mid\xt,\yy}[\xz]-\xhz(\xt)\big)\\[2pt]
&+ \nabla_{\xt}\log p(\xt).
\end{aligned}
\label{eq:app_boxed}
\end{equation}
This establishes Proposition~\ref{prop:score_general}.
\subsection{Gaussian closed form}
\label{app:closed_form_score2}
Under the Gaussian assumptions \eqref{eq:app_lik}--\eqref{eq:app_prior_cond}, the posterior distribution $p(\xz\mid\xt,\yy)$ is Gaussian and its expectation admits the closed form $\Exp_{\xz\mid\xt,\yy}[\xz]=\bm\mu_{\text{post}}$, given by \eqref{eq:mu_post}. The score \eqref{eq:app_boxed} then reads
\begin{equation}
\begin{aligned}
\nabla_{\xt}\log p(\xt\mid\yy) ={}& \frac{1}{r_t^2}\Big(\frac{\partial\xhz}{\partial\xt}\Big)^{\!\top}\big(\bm\mu_{\text{post}}-\xhz(\xt)\big)\\[2pt]
&+ \nabla_{\xt}\log p(\xt).
\end{aligned}
\label{eq:app_closed_mupost}
\end{equation}
We recover Corollary~\ref{cor:closed_form}.
\section{Connection with $\Pi$GDM}
\label{app:link}
To compare FB-GDM with $\Pi$GDM, let us fix the precisions instead of inferring them. We replace the Jeffreys priors \eqref{eq:M3}-\eqref{eq:M4} by Dirac priors, centered on the settings of $\Pi$GDM:
\begin{equation}
p(\gamma_r) = \delta\!\left(\gamma_r - \tfrac{1}{r_t^2}\right), \qquad
p(\gamma_b) = \delta\!\left(\gamma_b - \tfrac{1}{\sigma_b^2}\right).
\label{eq:dirac_priors}
\end{equation}
\begin{proposition}[Connection with $\Pi$GDM]
\label{prop:lien}
Under the Dirac priors \eqref{eq:dirac_priors}, the conditional score of FB-GDM reads
\begin{align}
&\nabla_{\xt}\log p_t(\xt \mid \yy) = \nabla_{\xt}\log p_t(\xt) \notag\\
&+ \Big(\tfrac{\partial \xhz}{\partial \xt}\Big)^{\!\top}\! \Ab^\top\!\left(\sigma_b^2\,\Iden + r_t^2\,\Ab\Ab^\top\right)^{-1}(\yy - \Ab\xhz),
\label{eq:lim_FB-GDM}
\end{align}
while that of $\Pi$GDM \eqref{eq:pigdm_score} reads
\begin{align}
&\nabla_{\xt}\log p_t(\xt \mid \yy) = \nabla_{\xt}\log p_t(\xt) \notag\\
&+ \frac{w\sqrt{\alpha_t\abart}\,r_t^2}{\beta_t}\Big(\tfrac{\partial \xhz}{\partial \xt}\Big)^{\!\top}\! \Ab^\top\!\left(\sigma_b^2\,\Iden + r_t^2\,\Ab\Ab^\top\right)^{-1}(\yy - \Ab\xhz).
\label{eq:lim_pigdm}
\end{align}
Both scores share the same measurement term $\big(\tfrac{\partial \xhz}{\partial \xt}\big)^{\!\top} \Ab^\top(\sigma_b^2\Iden + r_t^2 \Ab\Ab^\top)^{-1}(\yy - \Ab\xhz)$. They differ only by the scalar scale that weights it: exactly $1$ for FB-GDM, fixed by the Bayesian inference, against $w\sqrt{\alpha_t\abart}\,r_t^2/\beta_t$ for $\Pi$GDM, whose weight $w$ is tuned per task.
\end{proposition}
\begin{proof}
Under these priors, $\langle\gamma_r\rangle = 1/r_t^2$ and $\langle\gamma_b\rangle = 1/\sigma_b^2$. Rewriting $\bm\mu_{\text{post}}$ \eqref{eq:mu_full} with the matrix inversion lemma, the measurement term $\tfrac{1}{r_t^2}(\bm\mu_{\text{post}} - \xhz)$ of \eqref{eq:score_closed} becomes
\begin{equation}
\frac{1}{r_t^2}\big(\bm\mu_{\text{post}} - \xhz\big) = \Ab^\top\!\left(\sigma_b^2\,\Iden + r_t^2\,\Ab\Ab^\top\right)^{-1}(\yy - \Ab\xhz),
\label{eq:pushthrough}
\end{equation}
hence \eqref{eq:lim_FB-GDM}. For $\Pi$GDM, \eqref{eq:pigdm_score} weights the measurement term \eqref{eq:pigdm_grad} by the factor $w\sqrt{\alpha_t\abart}\,r_t^2/\beta_t$, which gives \eqref{eq:lim_pigdm}.
\end{proof}
%
% \balance
\bibliographystyle{IEEEtran}
\bibliography{references}
\begin{IEEEbiography}[{\includegraphics[width=1in,height=1.25in,clip,keepaspectratio]{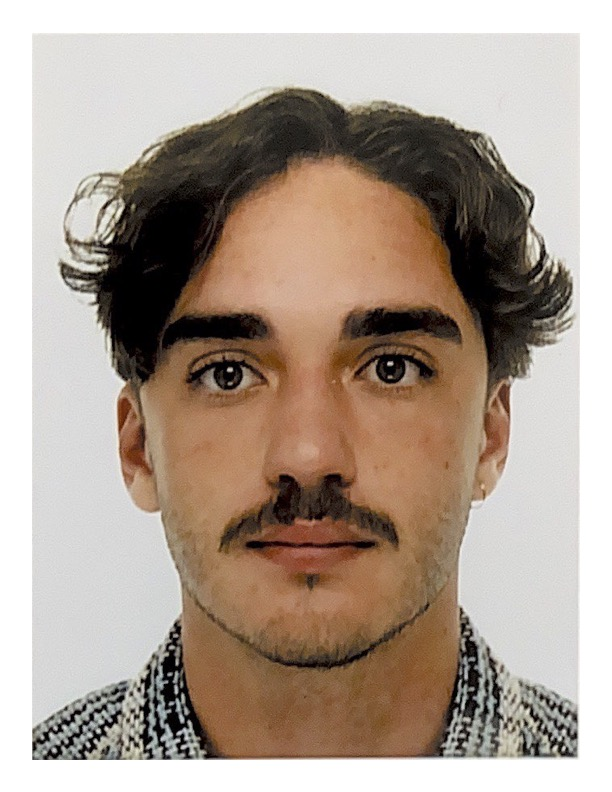}}]{Gatien Séguy}
was born in Rochefort, France, in 2002. He is currently pursuing the M.Sc. degree at the École Normale Supérieure Paris-Saclay, Université Paris-Saclay, Saclay, France.
\end{IEEEbiography}
\begin{IEEEbiography}[{\includegraphics[width=1in,height=1.25in,clip,keepaspectratio]{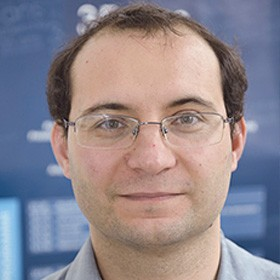}}]{Thomas Rodet}
was born in Lyon, France, in 1976. He received the Ph.D. degree from the Institut National Polytechnique de Grenoble, Grenoble, France, in 2002. He was an Assistant Professor with the University of Paris-Sud, Orsay, France, and a Researcher with the Laboratoire des Signaux et Systèmes (CNRS-Supelec-UPS), Gif-sur-Yvette, France, from 2003 to 2013. He is currently a Professor with the École Normale Supérieure Paris-Saclay, Saclay, France. His main research interests are tomography methods and Bayesian methods for inverse problems in astrophysical problems (inversion of data taken from space observatory: Spitzer, Herschel, SoHO, and STEREO).
\end{IEEEbiography}
\vfill
\end{document}